\documentclass[journal]{IEEEtran}
\usepackage{times}
\usepackage{helvet}
\usepackage{courier}
\usepackage{booktabs}
\usepackage{graphicx}
\usepackage{subfigure}
\usepackage{bm}
\usepackage{amsmath}
\usepackage{amsthm}
\usepackage{amssymb,amsfonts}
\usepackage{dsfont}
\usepackage{hyperref}

\newcommand{\argmin}{\operatornamewithlimits{arg\,min}}

\newtheorem{theorem}{Theorem}
\newtheorem{corollary}{Corollary}

\newtheorem{lemma}{Lemma}
\newtheorem{remark}{Remark}
\newtheorem{assumption}{Assumption}
\usepackage{multirow}
\usepackage[american]{babel}

\usepackage{algorithm}
\usepackage{algorithmic}
\usepackage{booktabs}

\begin{document}
%
\title{Max-Diversity Distributed Learning: \\Theory and Algorithms}
%
%
%

\author{Yong Liu,
        Jian Li,
        Weiping Wang
\thanks{
This work is supported in part by
the National Key Research and Development Program of China (No.2016YFB1000604),
the Science and Technology  Project of Beijing (No.Z181100002718004)
the National Natural Science Foundation of China (No.6173396, No.61673293, No.61602467)
 and the Excellent Talent Introduction of Institute of Information Engineering of CAS (Y7Z0111107).
 }
 \thanks{Y. Liu, J. Li and W.P. Wang are with the Institute of Information Engineering, Chinese Academy of Sciences
 (e-mail: liuyong@iie.ac.cn).
 }
}

\markboth{IEEE TRANSACTIONS ON NEURAL NETWORKS AND LEARNING SYSTEMS}%
{Shell \MakeLowercase{\textit{et al.}}: IEEE TRANSACTIONS ON NEURAL NETWORKS AND LEARNING SYSTEMS}

\maketitle

\begin{abstract}
 We study the risk performance of distributed learning for the regularization empirical
  risk minimization with fast convergence rate,
  substantially improving the error analysis of the existing divide-and-conquer based distributed learning.
  An interesting theoretical finding is that the larger the diversity of each local estimate is, the tighter the risk bound is.
  This theoretical analysis motivates us to devise an effective max-diversity distributed learning algorithm (MDD).
  Experimental results show that our proposed method can outperform the existing divide-and-conquer methods
  but with a bit more time.
  Theoretical analysis and empirical results demonstrate that our MDD  is sound and effective.
\end{abstract}


%
\IEEEpeerreviewmaketitle
\section{Introduction}
In the era of big data, the rapid expansion of computing capacities in automatic data generation
and acquisition brings data of unprecedented size and complexity, and raises a series
of scientific challenges such as storage bottleneck and algorithmic scalability \cite{zhou2014big,Zhang2013,lin2017distributed}.
Distributed learning based on a divide-and-conquer approach has triggered enormous
recent research activities in various areas such as optimization \cite{zhang2012communication}
data mining \cite{wu2014data} and machine learning \cite{gillick2006mapreduce}.
This learning strategy breaks up a big problem into manageable
pieces, operates learning algorithms on each piece on individual machines or
processors, and then puts the individual solutions together to get a final global output.
In this way, distributed learning is a feasible technique to conquer big data challenges.

This paper aims at error analysis of the distributed learning for (regularization) empirical risk minimization.
Given $$\mathcal{S}=\left\{z_i=(\mathbf  x_i,y_i)\right\}_{i=1}^N \in (\mathcal{Z}=\mathcal{X}\times \mathcal{Y})^N,$$
drawn identically and independently from a fixed,
but unknown probability  distribution $\mathbb{P}$ on
$\mathcal{Z}=\mathcal{X}\times\mathcal{Y}$,
the (regularization) empirical risk minimization can be stated as
\begin{align}
\label{def-empirical-f}
  \hat{f}=\argmin_{f\in\mathcal{H}} \hat{R}(f):=\frac{1}{N}\sum_{j=1}^N\ell(f,z_j)+r(f)
\end{align}
where $\ell(f,z)$ is a loss function, $r(f)$ is a regularizer, and $\mathcal{H}$ is a hypothesis space.
This learning algorithm  has been well studied in learning theory,
see e.g. \cite{de2005model,caponnetto2007optimal,steinwart2009optimal,smale2007learning,steinwart2008support}.
The distributed learning algorithm studied in this paper starts with partitioning the
data set $\mathcal{S}$ into $m$ disjoint subsets $\{\mathcal{S}_i\}_{i=1}^m$, $|\mathcal{S}_i|=\frac{N}{m}=:n$.
Then it assigns each data subset $\mathcal{S}_i$ to one
machine or processor to produce a local estimator $\hat{f}_i$:
\begin{align*}
  \hat{f}_i=\argmin_{f\in\mathcal{H}}\hat{R}_i(f):=
    \frac{1}{|\mathcal{S}_i|}\sum_{z_j\in\mathcal{S}_i}\ell(f,z_j)+r(f).
\end{align*}
The finally  global estimator $\bar{f}$ is synthesized by
$$\bar{f}=\frac{1}{m}\sum_{i=1}^m\hat{f}_i.$$

Theoretical foundations of distributed learning form a hot topic in machine learning
and have been explored recently in the framework of learning theory \cite{zhang2012communication,Zhang2013,lin2017distributed,guo2017learning}.
Under local strong convexity, smoothness and a reasonable set of other conditions, \cite{zhang2012communication} showed that the mean-squared error
decays as
\begin{align*}
  \mathbb{E}\left[\left\|\bar{f}-f^\ast\right\|^2\right]=\mathcal{O}\left(\frac{1}{N}+\frac{1}{n^2}\right),
\end{align*}
where $f^\ast$ is the optimal hypothesis in the hypothesis space.
Under some eigenfunction assumption,
the error analysis for distributed regularized
least squares in reproducing kernel Hilbert space (RKHS) was established in \cite{Zhang2013}:
if $m$ is not too large,
\begin{align*}
  \mathbb{E}\left[\left\|\bar{f}-f^\ast\right\|^2\right]=\mathcal{O}\left(\|f_\ast\|_\mathcal{H}^2+\frac{\gamma(\lambda)}{N}\right),
\end{align*}
where $\gamma(\lambda)=\sum_{j=1}^\infty\frac{\mu_j}{\lambda+\mu_j}$,
$\mu_j$ is the eigenvalue of a Mercer kernel function.
Without any eigenfunction assumption,
 an improved bound was derived for some $1\leq p\leq\infty$ \cite{lin2017distributed}:
\begin{align*}
  \mathbb{E}\left[\left\|\bar{f}-f^\ast\right\|\right]=
  \mathcal{O}\left(\left(\frac{\gamma(\lambda)}{N}\right)^{\frac{1}{2}(1-\frac{1}{p})}\left(\frac{1}{N}\right)^{\frac{1}{2p}}\right).
\end{align*}

There are two main contributions in this paper.
First, under strongly convex and smooth, and a reasonable set of other conditions,
we derive a risk bound of fast rate:
\begin{align}
\label{theorail-fast-rate}
    R(\bar{f})-R(f_\ast)=\mathcal{O}\left(\frac{H_\ast}{n}
    +\frac{1}{n^2}
    -\Delta_{\bar{f}}\right),
  \end{align}
  where  $$\Delta_{\bar{f}}=\mathcal{O}\left(\frac{1}{m^2}\sum_{i,j=1,i\not=j}^m\|\hat{f}_i-\hat{f}_j\|^2\right)$$ is the diversity between all partition-based estimates,
  $R(f)=\mathbb{E}_{z}[\ell(f,z)]+r(f),$ and
  $H_\ast=\mathbb{E}_{z}\left[\ell(f_\ast,z)\right]$.
When the minimal risk is small, i.e., $H_\ast=\mathcal{O}\left(\frac{1}{n}\right)$,
the rate is improved to
\begin{align*}
    R(\bar{f})-R(f_\ast)=\mathcal{O}\left(\frac{1}{n^2}-\Delta_{\bar{f}}\right).
\end{align*}
Thus, if $m\leq \sqrt{N}$, the order of $R(\bar{f})-R(f_\ast)$ is faster than
$\mathcal{O}\left(\frac{1}{N}-\Delta_{\bar{f}}\right).$
Note that if $\ell(f,z)+r(f)$ is $L$-Lipschitz continuous over $f$,
the order of $R(\bar{f})-R(f^\ast)$ is
\begin{align*}
  R(\bar{f})-R(f^\ast)&=\mathcal{O}\left(L\mathbb{E}\left[\left\|\bar{f}-f^\ast\right\|\right]\right)\\&=\mathcal{O}\left(L\sqrt{\mathbb{E}\left[\left\|\bar{f}-f^\ast\right\|^2\right]}\right).
\end{align*}
Thus, the order of $R(\bar{f})-R(f^\ast)$ in \cite{zhang2012communication,Zhang2013,lin2017distributed}
 at most $\mathcal{O}\big({\frac{1}{\sqrt{N}}}\big)$,
 which is much slower than that of our bound.
Our second contribution is to develop a novel max-diversity distributed learning algorithm.
From Equation \eqref{theorail-fast-rate},
we know that the larger the diversity $\Delta_{\bar{f}}$ is, the tighter the risk bound is.
This  interesting theoretical finding motivates us to devise a max-diversity distributed learning algorithm (MDD):
\begin{align}
\label{equation-op}
  \hat{f}_i=\argmin_{f\in\mathcal{H}}\frac{1}{|\mathcal{S}_i|}\sum_{z_j\in\mathcal{S}_i}\ell(f,z_j)+r(f)-\gamma \|f-\bar{f}_{\backslash i}\|_\mathcal{H},
\end{align}
where $$\bar{f}_{\backslash i}=\frac{1}{m-1}\sum_{j=1,j\not =i}^m\hat{f}_j.$$
The last term of \eqref{equation-op} is to make $\Delta_{\bar{f}}$ large.
Experimental results on lots of datasets show that our proposed MDD is sound and efficient.

The notion of diversity is popular used in ensemble learning to improve the performance.
But to the best of our knowledge,
this is the first time that theoretical results w.r.t. diversity are given for a distributed setting.

The rest of the paper is organized as follows.
In Section 2, we derive a risk bound of distributed learning with fast convergence rate.
In Section 3, we  propose two novel algorithms based on the max-diversity of each local estimate in linear space and RKHS.
In Section 4, we empirically analyze the performance of our proposed algorithms.
We end in Section 5 with conclusion.
All the proofs are given in the last part.

\section{Error Analysis of Distributed Learning}
In this section, we will derive a sharper risk bound under some common assumptions.
\subsection{Assumptions}
In the following, we use $\|\cdot\|_\mathcal{H}$ to denote the norm induced by inner product of the Hilbert space $\mathcal{H}$.
Let the expected risk  $R(f)$ and $f_\ast$ be
\begin{align*}
  R(f)=\mathbb{E}_{z}[\ell(f,z)]+r(f) \text{ and } f_\ast=\argmin_{f\in\mathcal{H}}R(f).
\end{align*}
\begin{assumption}
\label{assumption-strongly}
  The risk $R(f)$ is an $\eta$-strongly convex function,
  that is $\forall f,f'\in\mathcal{H}$,
  \begin{align}
    \label{assumption-strongly-equation}
     \langle \nabla R(f), f-f'\rangle_\mathcal{H}+\frac{\eta}{2}\|f-f'\|_\mathcal{H} &\leq R(f)-R(f'),
  \end{align}
  or (another equivalent definition)
  $\forall  f,f'\in\mathcal{H}, t\in[0,1]$,
  \begin{align}
  \label{assumption-strongly-second}
  \begin{aligned}
  &R(tf+(1-t)f')\\ &\leq  tR(f)+(1-t)R(f')-\frac{1}{2}\eta t(t-1)\|f-f'\|_\mathcal{H}^2.
  \end{aligned}
  \end{align}
\end{assumption}
\begin{assumption}
\label{assumption-strongly-emprical}
  The empirical risk $\hat{R}(f)$  is a convex function.
\end{assumption}
\begin{assumption}
\label{assumption-smooth-loss}
  The loss function $\ell(f,z)$ is $\tau$-smooth with respect to the first variable $f$,
  that is $\forall f,f'\in\mathcal{H}$,
  \begin{align}
     \label{assumption-smooth-equaiton-loss}
     \left\|\nabla\ell(f,\cdot)-\nabla \ell(f',\cdot)\right\|_\mathcal{H}&\leq \tau\|f-f'\|_\mathcal{H}.
  \end{align}
\end{assumption}
\begin{assumption}
\label{assumption-smooth-r}
  The regularizer $r(f)$ is a $\tau'$-smooth function,
  that is $\forall f,f'\in\mathcal{H}$,
  \begin{align}
     \label{assumption-smooth-equaiton-loss}
     \left\|\nabla r(f)-\nabla r(f')\right\|_\mathcal{H}&\leq \tau'\|f-f'\|_\mathcal{H}.
  \end{align}
\end{assumption}

\begin{assumption}
\label{assumption-libs}
  The function $\nu(f,z)=\ell(f,z)+r(f)$ is $L$-Lipschitz continuous with respect to the first variable $f$,
  that is $\forall f,f'\in\mathcal{H}$,
  \begin{align}
     \label{assumption-libs-equation}
     \left\|\nu(f,\cdot)- \nu(f',\cdot)\right\|_\mathcal{H}&\leq L\|f-f'\|_\mathcal{H}.
  \end{align}
\end{assumption}

\textbf{Assumptions} \ref{assumption-strongly}, \ref{assumption-strongly-emprical}, \ref{assumption-smooth-loss}, \ref{assumption-smooth-r} and \ref{assumption-libs} allow us to model some popular losses,
such as square loss and logistic loss, and some regularizer, such as $r(f)=\lambda \|f\|_\mathcal{H}^2$.

\begin{assumption}
\label{assumption-optimal-bound}
  We assume that the gradient at $f_\ast$ is upper bounded by $M$, that is
  \begin{align*}
    \|\nabla \ell(f^\ast,\cdot)\|_\mathcal{H}\leq M.
  \end{align*}
\end{assumption}
Assumption \ref{assumption-optimal-bound} is also a common assumption, which is used in \cite{Zhang2017er,zhang2012communication}.
\subsection{Faster Rate of Distributed Learning}
Let $\mathcal{N}(\mathcal{H},\epsilon)$ be the $\epsilon$-net of $\mathcal{H}$ with minimal cardinality,
and $C(\mathcal{H},\epsilon)$ the covering number of $|\mathcal{N}(\mathcal{H},\epsilon)|$

\begin{theorem}
\label{theorem-main}
For any $0<\delta<1$, $\epsilon\geq 0$,
under \textbf{Assumptions}  \ref{assumption-strongly}, \ref{assumption-strongly-emprical},
\ref{assumption-smooth-loss}, \ref{assumption-smooth-r}, \ref{assumption-libs} and \ref{assumption-optimal-bound},
and when
  \begin{align}
    \label{equation-12}
    m\leq \frac{N\eta}{4\tilde{\tau}\log C(\mathcal{H},\epsilon)},
  \end{align}
  with probability at least $1-\delta$,
  we have
  \begin{align}
    \label{equation-13}
    \begin{aligned}
    &~~~R(\bar{f})-R(f_\ast)\\&\leq
    \frac{16\tilde{\tau} \log(4m/\delta)}{n^2\eta}+\frac{128\tau H_\ast\log(4m/\delta)}{n\eta}\\ &~~~+ \frac{32\tilde{\tau}^2\epsilon^2}{\eta}+
    \frac{64\tilde{\tau} L \epsilon \log C(\mathcal{H},\epsilon)}{n\eta}\\
   &~~~+\frac{64\tilde{\tau} \epsilon^2 \log^2C(\mathcal{H},\epsilon)}{n^2\eta}
   -\Delta_{\bar{f}},
  \end{aligned}
  \end{align}
  where $\Delta_{\bar{f}}=\frac{\eta}{4m^2}\sum_{i,j=1,i\not=j}^m\|\hat{f}_i-\hat{f}_j\|_\mathcal{H}^2$, $H_\ast=\mathbb{E}_{z}\left[\ell(f_\ast,z)\right]$ and $\tilde{\tau}=\tau+\tau'$.
\end{theorem}
From the above theorem, an  interesting finding is that,
when the larger the diversity of each local estimate is,
the tighter the risk bound is.
Furthermore, one can also see that when $\epsilon$ small enough,
$$\frac{32\tilde{\tau}^2\epsilon^2}{\eta}+
    \frac{64\tilde{\tau} L \epsilon \log C(\mathcal{H},\epsilon)}{n\eta}
    +\frac{64\tilde{\tau} \epsilon^2\log^2C(\mathcal{H},\epsilon)}{n^2\eta}$$
will become non-dominating.
To be specific, we have the following corollary:
\begin{corollary}
\label{corollary-first}
  By setting $\epsilon=\frac{1}{n}$ in Theorem \ref{theorem-main},
  when $m\leq \frac{N\eta}{4\tilde{\tau}\log C(\mathcal{H},1/n)}$,
  with high probability,
  we have
  \begin{align*}
    R(\bar{f})-R(f_\ast)=\mathcal{O}\left(\frac{H_\ast\log(m)}{n}
    +\frac{\log(\mathcal{N}(\mathcal{H},\frac{1}{n}))}{n^2}
    -\Delta_{\bar{f}}\right).
  \end{align*}
\end{corollary}
If the the minimal risk $H_\ast$ is small, i.e., $H_\ast=\mathcal{O}(\frac{1}{n})$,
the rate can reach $$\mathcal{O}\left(\frac{\log(m)}{n^2}
    +\frac{\log(\mathcal{N}(\mathcal{H},\frac{1}{n}))}{n^2}
    -\Delta_{\bar{f}}\right).$$
To the best of our knowledge,
this is the first $\tilde{\mathcal{O}}\left(\frac{1}{n^2}\right)$-type of distributed
risk bound for (regularization) empirical risk minimization.

In the next, we will consider two popular hypothesis spaces:
linear and reproducing kernel Hilbert space (RKHS).
\subsection{Linear Space}
\label{subsection-3.1}
The linear hypothesis space we considered is defined as
\begin{align*}
\mathcal{H}=\left\{f=\mathbf w^\mathrm T\mathbf x\Big|\mathbf w\in \mathbb{R}^d, \|\mathbf w\|_2\leq B\right\}.
\end{align*}
From \cite{pisier1999volume},
the cover number of linear hypothesis space can be bounded by
\begin{align*}
  \log\left(C(\mathcal{H},\epsilon)\right)\leq d\log \left(\frac{6B}{\epsilon}\right).
\end{align*}
Thus, if we set $\epsilon=\frac{1}{n}$, from Corollary \ref{corollary-first}, we have
\begin{align*}
  R(\bar{f})-R(f_\ast)&=\mathcal{O}\left(\frac{H_\ast\log m}{n}+\frac{d\log n}{n^2}-
  \Delta_{\bar{f}}\right)
\end{align*}
When the minimal risk is small, i.e., $H_\ast=\mathcal{O}\left(\frac{d}{n}\right)$,
the rate is improved to
\begin{align*}
    \mathcal{O}\left(\frac{d\log (mn)}{n^2}-\Delta_{\bar{f}}\right)=\mathcal{O}\left(\frac{d\log N}{n^2}-\Delta_{\bar{f}}\right).
\end{align*}
Therefore, if $m\leq \sqrt{\frac{N}{d\log N}}$, the order of risk bound can even faster than
$\mathcal{O}\left(\frac{1}{N}\right).$
\subsection{Reproducing Kernel Hilbert Space (RKHS)}
\label{subsection-3.2}
The reproducing kernel Hilbert space $\mathcal{H}_K$ associated with the kernel $K$ is
defined to be the closure of the linear span of the set of functions
$\left\{K(\mathbf x,\cdot):\mathbf x\in\mathcal{X}\right\}$ with the inner product satisfying
\begin{align*}
  \langle K(\mathbf x,\cdot), f\rangle_{K}=f(\mathbf x), \forall \mathbf x\in\mathcal{X}, f\in\mathcal{H}_K.
\end{align*}

The  hypothesis space of the reproducing kernel Hilbert space we considered in this paper is
\begin{align*}
  \mathcal{H}:=\left\{f\in\mathcal{H}_{K}: \|f\|_{\mathcal{H}_K}\leq B\right\}.
\end{align*}

From \cite{zhou2002covering}, if the kernel function $K$ is the popular Gaussian kernel over $[0,1]^d$:
$$
  K(\mathbf x,\mathbf x')=\exp\left\{-\frac{\|\mathbf x-\mathbf x'\|^2}{\sigma^2}\right\}, \mathbf x,\mathbf x' \in[0,1]^d,
$$
then for $0\leq \epsilon\leq \frac{B}{2}$,
$$
 \log \left(C(\mathcal{H},\epsilon)\right)=\mathcal{O}\left(\log^d\left(\frac{B}{\epsilon}\right)\right).$$
From Corollary \ref{corollary-first}, if we set $\epsilon=\frac{1}{n}$, and assume $R_\ast=\mathcal{O}\left(\frac{1}{n}\right)$,
we have
\begin{align*}
  R(\bar{f})-R(f_\ast)=\mathcal{O}\left(\frac{\log m}{n^2}+\frac{\log^d n}{n^2}-
  \Delta_{\bar{f}}\right)
\end{align*}
Therefore, if $m\leq \min\left\{\sqrt{\frac{N}{d\log N}}, \sqrt{\frac{N}{\log^d n}}\right\}$,
the order is faster than $\mathcal{O}\left(\frac{1}{N}\right)$.

\subsection{Comparison with Related Work}
In this subsection, we will compare our bound with the related work \cite{zhang2012communication,Zhang2013,lin2017distributed}.
Under the smooth, strongly convex and other some assumptions,
a distributed risk bound is given in \cite{zhang2012communication}:
\begin{align*}
  \mathbb{E}\left[\|\bar{f}-f_\ast\|^2\right]=\mathcal{O}\left(\frac{1}{N}+\frac{\log d}{n^2}\right).
\end{align*}
Under some eigenfunction assumption,
the error analysis for distributed regularized
least squares were established in \cite{Zhang2013},
\begin{align*}
  \mathbb{E}\left[\left\|\bar{f}-f^\ast\right\|^2\right]=\mathcal{O}\left(\|f_\ast\|_\mathcal{H}^2+\frac{\gamma(\lambda)}{N}\right).
\end{align*}
By removing the eigenfunction assumptions with a novel integral operator method of \cite{Zhang2013},
 a new bound was derived \cite{lin2017distributed}:
\begin{align*}
  \mathbb{E}\left[\left\|\bar{f}-f^\ast\right\|\right]=
  \mathcal{O}\left(\left(\frac{\gamma(\lambda)}{N}\right)^{\frac{1}{2}(1-\frac{1}{p})}\left(\frac{1}{N}\right)^{\frac{1}{2p}}\right).
\end{align*}
Note that, if $\nu(f,z)$ is $L$-Lipschitz continuous over $f$, that is
\begin{align*}
  \forall f, f\in \mathcal{H}, z\in\mathcal{Z}, |\nu(f,z)-\nu(f',z)|\leq L\|f-f'\|,
\end{align*}
we can obtain that
\begin{align}
  \nonumber R(f)-R(f_\ast)&\leq L\mathbb{E}\left[\|\bar{f}-f_\ast\|\right]\leq L\sqrt{\mathbb{E}\left[\|\bar{f}-f_\ast\|^2\right]}
\end{align}
Thus, the order of \cite{Zhang2013,lin2017distributed,zhang2012communication} of  $R(f)-R(f_\ast)$ is at most $\mathcal{O}\left(\frac{1}{\sqrt{N}}\right).$

According to the subsections \ref{subsection-3.1} and \ref{subsection-3.2},
if $m$ is not very large, and $H_\ast$ is small,
the order of this paper can even faster than $\mathcal{O}\left(\frac{1}{N}\right)$,
which is much faster than those of in the related work \cite{zhang2012communication,Zhang2013,lin2017distributed}.
\section{Max-Discrepant Distributed Learning (MDD)}
In this section, we will propose two novel algorithms for linear space and RKHS.
From corollary \ref{corollary-first},  we know that
  $$
     R(f)-R(f_\ast)=\mathcal{O}\left(\frac{1}{n^2}-\frac{1}{m^2}\sum_{i,j=1,i\not=j}^m\|\hat{f}_i-\hat{f}_j\|_\mathcal{H}^2\right).
$$
Thus, to obtain tighter bound, the diversity of each local estimate $\hat{f}_i, i=1,\ldots,m$, should be larger.
\subsection{Linear Hypothesis Space}
When $\mathcal{H}$ is a linear Hypothesis space,
we consider the following optimization problem:
\begin{align}
 \label{optimation-linear-space}
  \hat{\mathbf w}_i=\argmin_{\mathbf w\in\mathbb{R}^d}
  \frac{1}{n}\sum_{z_i\in\mathcal{S}_i}(\mathbf w^\mathrm{T}\mathbf x_i-y_i)^2+\lambda \|\mathbf w\|_2^2+ \gamma \mathbf w^\mathrm{T}\bar{\mathbf w}_{\backslash i},
\end{align}
where $\bar{\mathbf w}_{\backslash i}=\frac{1}{m-1}\sum_{j=1,j\not =i}\hat{\mathbf w}_j$.
Note that, if given $\bar{\mathbf w}_{\backslash i}$,  $\hat{\mathbf w}_i$ has following closed form solution:
\begin{align*}
  \hat{\mathbf w}_i=\Big(\underbrace{\frac{1}{n}\mathbf X_{\mathcal{S}_i}\mathbf X_{\mathcal{S}_i}^\mathrm{T}+\lambda \mathbf I_d}_{:=\mathbf A_i}\Big)^{-1}
  \Big(\underbrace{\frac{1}{n}\mathbf X_{\mathcal{S}_i}\mathbf y_{\mathcal{S}_i}}_{:=\mathbf b_i}- \frac{\gamma\bar{\mathbf w}_{\backslash i}}{2}\Big),
\end{align*}
where $\mathbf X_{\mathcal{S}_i}=(\mathbf x_{t_1},\mathbf x_{t_2},\ldots, \mathbf x_{t_n})$,
$\mathbf y_{\mathcal{S}_i}=(y_{t_1},y_{t_2},\ldots,y_{t_n})^\mathrm{T}$, $z_{t_j}\in \mathcal{S}_i$, $j=1,\ldots, n$.
In the next, we will give an iterative algorithm to
solve the optimization problem \eqref{optimation-linear-space}.
In each iteration, we should compute $\mathbf A_i^{-1}\bar{\mathbf w}_{\backslash i}$,
which needs $\mathcal{O}\left(d^2\right)$ if given $\mathbf A_i^{-1}$,
which is computational intensive.
Fortunately, 
from Lemma 4  (see in supplementary material), the $\mathbf A_i^{-1}\bar{\mathbf w}_{\backslash i}$ can be computed by
\begin{align*}
  \mathbf A_i^{-1}\bar{\mathbf w}_{\backslash i}=
  \left(\bar{\mathbf w}_{\backslash i}^\mathrm{T}\mathbf c_i\right)./\mathbf b_i, \mathbf c_i=\mathbf A_i^{-1}\mathbf b_i
\end{align*}
where $a./\mathbf c=(a/c_1,\ldots a/c_d)^\mathrm{T}$, which only needs $\mathcal{O}(d)$.

The Max-Discrepant Distributed Learning  algorithm for linear space is given in Algorithm \ref{alg:RMMls}.
Compared with the traditional divide-and-conquer method,
our \texttt{MDD} for linear space only need add $\mathcal{O}(d)$ in each iteration for each worker node.

\begin{algorithm}
    \caption{Max-Discrepant Distributed Learning for Linear Space (\texttt{MDD-LS})}
    \label{alg:RMMls}
    \begin{algorithmic}[1]
    \STATE \textbf{Input}: $\lambda,\gamma$, $\mathbf X$, $m$, $\zeta$.
     \STATE \emph{For each worker node $i$:} $\hat{\mathbf w}_i^0=\mathbf A_i^{-1} \mathbf b_i$, and push $\hat{\mathbf w}_i^0$ to the server node.\\
         ~~~~~~~~// $\mathbf A_i=\frac{1}{n}\mathbf X_{\mathcal{S}_i}\mathbf X_{\mathcal{S}_i}^\mathrm{T}+
            \lambda \mathbf I_d$, $\mathbf b_i= \frac{1}{n}\mathbf X_{\mathcal{S}_i}\mathbf y_{\mathcal{S}_i}$.
    \STATE \emph{For server node:}
    $\bar{\mathbf w}^0=\frac{1}{m}\sum_{i=1}^m\hat{\mathbf w}_i^0$,
    $\bar{\mathbf w}^{0}_{\backslash i}=\frac{m\bar{\mathbf w}^{0}-\hat{\mathbf w}_i^0}{m-1}$.
    \FOR{$t=1,2,\ldots$}
    \STATE  \emph{For each worker node $i$:} \\
    ~~~~Pull $\bar{\mathbf w}^{t-1}_{\backslash i}$ from server node.
    \STATE ~~~~$\mathbf d_i^t=\left(\left(\bar{\mathbf w}^{t-1}_{\backslash i}\right)^\mathrm{T}\hat{\mathbf w}_i^{0}\right)./\mathbf b_i$,
     $\hat{\mathbf w}_i^t=\hat{\mathbf w}_i^0-\gamma\mathbf d_i^t$.
    \STATE ~~~~Push $\hat{\mathbf w}_i^t$ to the server node.
     \STATE \emph{For server node:}
     \STATE ~~~~$\bar{\mathbf w}^t=\frac{1}{m}\sum_{i=1}^m\hat{\mathbf w}_i^t$\\
      ~~~~\textbf{if} {$\|\bar{\mathbf w}^{t}-\bar{\mathbf w}^{t-1}\|\leq \zeta$} \textbf{end for}
      \STATE ~~~~\textbf{else} $\bar{\mathbf w}^{t}_{\backslash i}=\frac{m\bar{\mathbf w}^{t}-\hat{\mathbf w}_i^t}{m-1}$.
    \ENDFOR
    \STATE \textbf{Output}: $\bar{\mathbf w}=\frac{1}{m}\sum_{i=1}^m\hat{\mathbf w}_i^t$.
    \end{algorithmic}
\end{algorithm}


\subsection{Reproducing Kernel Hilbert Space}
When $\mathcal{H}$ is a reproducing kernel Hilbert space, that is $f(\mathbf x)=\sum_{j=1}^n w_j K(\mathbf x_j,\mathbf x)$,
we consider the following optimization problem:
\begin{align}
\begin{aligned}
\label{opti-RKHS}
  \hat{\mathbf w}_i=\argmin_{\mathbf w\in\mathbb{R}^n}
  &\frac{1}{n}\|\mathbf K_{\mathcal{S}_i}\mathbf w-\mathbf y_{\mathcal{S}_i}\|_2^2+\lambda \mathbf w^\mathrm{T}\mathbf K_{\mathcal{S}_i}\mathbf w
  \\&+\frac{\gamma}{m-1} \sum_{j=1,j\not=i}^{m}\mathbf w^\mathrm{T}\mathbf K_{\mathcal{S}_i}\mathbf K_{\mathcal{S}_i,\mathcal{S}_j}\hat{\mathbf w}_j,
\end{aligned}
\end{align}
where $\mathbf K_{\mathcal{S}_i}=\Big[K(\mathbf x_{t_j},\mathbf x_{t_{j'}})\Big]_{j,j'=1}^n$,
$z_{t_j},z_{t_{j'}}\in\mathcal{S}_i$,
$\mathbf K_{\mathcal{S}_i,\mathcal{S}_j}=\Big[K(\mathbf x_{t_j},\mathbf x_{t_{k}})\Big]_{j,k=1}^n$,
$z_{t_j}\in\mathcal{S}_i,z_{t_{k}}\in\mathcal{S}_j$.
Note that $\hat{\mathbf w}_i$ can be written as
\begin{align*}
  \hat{\mathbf w}_i=\Big(\underbrace{\frac{1}{n}\mathbf K_{\mathcal{S}_i}+\lambda \mathbf I_n}_{:=\mathbf A_i}\Big)^{-1}
  \Big(\underbrace{\frac{1}{n}\mathbf y_{\mathcal{S}_i}}_{:=\mathbf b_i}- \frac{\gamma}{2}\bar{\mathbf g}_{\backslash i}\Big).
\end{align*}
where $\mathbf g_j=\mathbf K_{\mathcal{S}_i,\mathcal{S}_j}\hat{\mathbf w}_j$ and $\bar{\mathbf g}_{\backslash i}=\frac{1}{m-1}\sum_{j=1,j\not=i}^m \hat{\mathbf g}_j$.

Similar with the linear space, we need to compute $\mathbf A_i^{-1}\bar{\mathbf g}_{\backslash i}$ in each iterative.
From Lemma 4 (see in supplementary material), we know that
\begin{align*}
  \mathbf A_i^{-1}\bar{\mathbf g}_{\backslash i}=
  \left(\bar{\mathbf g}_{\backslash i}^\mathrm{T}\mathbf c_i\right)./\mathbf b_i,  \mathbf c_i=\mathbf A_i^{-1}\mathbf b_i.
\end{align*}
The Max-Discrepant Distributed Learning algorithm for RKHS is given in Algorithm \ref{alg:RMMRKHS}.
Compared with the traditional divide-and-conquer method,
our \texttt{MDD} for RKHS  only need add $\mathcal{O}(n)$ in each iteration for local machine.

\begin{algorithm}
    \caption{Max-Discrepant Distributed Learning for RKHS (\texttt{MDD-RKHS})}
    \label{alg:RMMRKHS}
    \begin{algorithmic}[1]
    \STATE \textbf{Input}: $\lambda,\gamma$, kernel function $K$, $\mathbf X$, $m$, $\zeta$.
     \STATE \emph{For each worker node $i$:} $\hat{\mathbf w}_i^0=\mathbf A_i^{-1} \mathbf b_i$, and push $\hat{\mathbf w}_i^0$ to the server node.\\
            ~~~~~~~~// $\mathbf A_i=\frac{1}{n}\mathbf K_{\mathcal{S}_i}+
            \lambda \mathbf I_n$, $\mathbf b_i= \frac{1}{n}\mathbf y_{\mathcal{S}_i}$.
    \STATE \emph{For server node:} $\hat{\mathbf g}_{i,j}^0=\mathbf K_{\mathcal{S}_i,\mathcal{S}_j}\hat{\mathbf w}_j^0$, $i, j=1,\ldots,m$, $\bar{\mathbf g}^{0}_{\backslash i}=\frac{m\bar{\mathbf g}_i^{0}-\hat{\mathbf g}_i^0}{m-1}$.
    \FOR{$t=1,2,\ldots$}
    \STATE  \emph{For each worker node $i$:}
    \STATE ~~~~Pull $\bar{\mathbf g}^{t-1}_{\backslash i}$ from server node.
    \STATE ~~~~$\mathbf d_i^t=\left(\left(\bar{\mathbf g}^{t-1}_{\backslash i}\right)^\mathrm{T}\hat{\mathbf w}_i^0\right)./\mathbf b_i$,
     $\hat{\mathbf w}_i^t=\hat{\mathbf w}_i^0-\gamma\mathbf d_i^t$.
    \STATE ~~~~Push $\hat{\mathbf w}_i^t$ to the server node.
     \STATE \emph{For server node:}
     \STATE ~~~~$\hat{\mathbf g}_{i,j}^{t}=\mathbf K_{\mathcal{S}_i,\mathcal{S}_j}\hat{\mathbf w}_j^t$, $i, j=1,\ldots,m$,
     $\bar{\mathbf g}_{i}^t=\frac{1}{m}\sum_{j=1}^m \hat{\mathbf g}_{i,j}^t$.\\
      ~~~~\textbf{if} {$\frac{1}{m}\sum_{i=1}^m\|\bar{\mathbf g}_i^{t}-\bar{\mathbf g}_i^{t}\|\leq \zeta$} \textbf{end for}
      \STATE ~~~~\textbf{else}
      \STATE ~~~~~~~~$\bar{\mathbf g}^{t}_{\backslash i}=\frac{m\bar{\mathbf g}_i^{t}-\hat{\mathbf g}_i^t}{m-1}$.
    \ENDFOR
    \STATE \textbf{Output}: $\bar{f}=\frac{1}{m}\sum_{i=1}^m\hat{f}_i$, where $\hat{f}_i=\mathbf k_{\mathcal{S}_i}^\mathrm{T}\hat{\mathbf w}_i$,
    where $\mathbf k_{\mathcal{S}_i}=(K(\mathbf x_1,\cdot),\ldots,K(\mathbf x_n,\cdot))^\mathrm{T}$, $z_j\in\mathcal{S}_i$
    \end{algorithmic}
\end{algorithm}
\begin{remark}
The motivation of this paper was inspired by the ensemble learning,
but one more thing should be emphasized, the theoretical proof and algorithm design of this paper are not from the ensemble learning.
\end{remark}

   \begin{table*}[t]
   \caption{
     Comparison of average root mean square error of our \texttt{MDD-LS} and \texttt{MDD-RKHS} with
    \texttt{RR}, \texttt{DRR}, \texttt{KRR}, \texttt{DKRR}.
     We bold the numbers of the best method and underline the numbers of the other methods
    which are not significantly worse than the best one.
   }
   \label{tabel:mse}
    \begin{tabular*}{\linewidth}{@{\extracolsep{\fill}}lccccccccc}
    \toprule
                                &madelon                  &space\_ga               &cpusmall            &phishing           &cadata             &a8a                  &a9a                    &cod\-rna                   &YearPred                 \\   \hline
\texttt{RR}                              &\textbf{0.971}           &\textbf{2.585}          &\textbf{45.150}     &\textbf{0.247}     &\textbf{1.932}     &\textbf{0.671}       &\textbf{0.673}         &\textbf{0.841}             &\textbf{12.233} \\
\texttt{DRR-5}                           &\underline{0.989}                    &2.814                   &53.114              &0.262              &2.659              &\underline{0.681}    &\underline{0.680}      &0.855                      &14.216 \\
\texttt{DRR-10}                          &1.408                    &2.983                   &55.557              &0.273              &2.839              &0.725                &0.696                  &0.863                      &15.780 \\
\texttt{MDD-LS-5}                        &\underline{0.977}        &\underline{2.677}                   &\underline{46.184}              &\underline{0.257}              &\underline{2.114}  &\underline{0.677}    &\underline{0.673}      &\underline{0.847}                      &\underline{12.303}\\
\texttt{MDD-LS-10}                       &\underline{1.021}                    &\underline{2.750}                  &\underline{47.956}              &0.268              &\underline{2.352}              &\underline{0.703}                &\underline{0.685}                  &\underline{0.854}                      &14.158\\
\hline \hline
\texttt{KRR}                            &\textbf{0.959}           &\textbf{1.458}          &\textbf{43.993}     &\textbf{0.167}     &\textbf{1.504}     &\textbf{0.659}       &\textbf{0.630}         &\textbf{0.651}             &/ \\
\texttt{KDRR-5}                          &\underline{1.142}                    &2.389                   &\underline{44.228}  &0.419              &\underline{1.598}              &0.873                &\underline{0.666}                  &\underline{0.674}                      &\underline{5.397}\\
\texttt{KDRR-10}                         &1.374                    &2.531                   &46.233              &0.422              &1.824              &0.906                &0.893                  &0.707                      &5.631\\
\texttt{MDD-RKHS-5}                      &\underline{0.992}        &2.030                   &\underline{44.015}  &0.214              &\underline{1.554}  &\underline{0.745}                &\underline{0.604}      &\underline{0.672}          &\textbf{5.350}\\
\texttt{MDD-RKHS-10}                     &\underline{1.192}                    &2.326                   &\underline{45.120}              &0.239              &1.780              &\underline{0.673}                &\underline{0.649}                  &\underline{0.683}                      &\underline{5.534}\\

   \bottomrule
\end{tabular*}
\end{table*}
\subsection{Complexity}
\textbf{Linear space}: At the very beginning, we need $\mathcal{O}\left(nd^2\right)$
to compute the $\mathbf A_i$,  $\mathcal{O}(d^3)$ to compute $\mathbf A_i^{-1}$ for each worker node.
In each iteration, worker nodes cost $\mathcal{O}(d)$ to compute $\mathbf d^t_i$ and
the server node costs $O(md)$ to compute $\bar{\mathbf w}^{t}_{\backslash i}$.
So, the sequential computation complexity is $\mathcal{O}\left(nd^2+d^3+Tmd\right)$, where $T$ is the number of iteration.
Moreover, the total communication complexity is $O(Td)$.

\textbf{RKHS}: At the very beginning, we need $\mathcal{O}\left(n^2d\right)$ to compute the $\mathbf A_i$
and $\mathcal{O}(n^3)$ to compute $\mathbf A_i^{-1}$.
In each iteration, worker nodes cost $\mathcal{O}(n)$ to compute $\mathbf d^t_i$
and the server node costs $O(mn)$ to compute $\bar{\mathbf g}^{t}_{\backslash i}$.
So, the sequential computation complexity is $\mathcal{O}\left(n^2d+n^3+Tmn\right)$, where $T$ is the number of iteration.
Moreover, the total communication complexity is $O(Tn)$.

\textbf{Divide-and-conquer approach}: The sequential complexities of linear space and RKHS
are $\mathcal{O}\left(nd^2+d^3\right)$ and $\mathcal{O}\left(n^2d+n^3\right)$, respectively.
Meanwhile, the communication complexities are $O(d)$ and $O(n)$.

\textbf{Global approach}: The total complexities of linear space and RKHS are $\mathcal{O}\left(Nd^2+d^3\right)$
and $\mathcal{O}\left(N^2d+N^3\right)$, respectively.

\section{Experiments}
In this section, we will compare our  \texttt{MDD} methods with the global method and divide-and-conquer method in both Linear and RKHS Hypothesis.
Actually, we compare six approaches: global Ridge Regression (\texttt{RR}) \cite{hoerl1970ridge},
divide-and-conquer Ridge Regression (\texttt{DRR}) and our \texttt{MDD-LS} (Algorithm \ref{alg:RMMls}) in Linear Hypothesis Space,
meanwhile, global Kernel Ridge Regression (\texttt{KRR}) \cite{an2007fast},
divide-and-conquer Kernel Ridge Regression (\texttt{KDRR}) \cite{Zhang2013} and our \texttt{MDD-RKHS} (Algorithm \ref{alg:RMMRKHS}) in Reproducing Kernel Hilbert Space.
Based on the recent distributed machine learning platform PARAMETER SERVER \cite{li2014scaling},
we implemented divide-and-conquer methods and \texttt{MDD} methods and do experiments on this framework.

\begin{table*}[t]
   \caption{
     Comparison of run time (second) amound our proposed \texttt{MDD-LS} and \texttt{MDD-RKHS} with other methods.
   }
   \label{tabel:time}
    \begin{tabular*}{\linewidth}{@{\extracolsep{\fill}}lccccccccc}
    \toprule
                                &madelon                  &space\_ga               &cpusmall            &phishing           &cadata             &a8a                  &a9a                    &cod\-rna                   &YearPred                \\   \hline
\texttt{RR}                              &2.069                    &0.280                   &1.218               &1.526              &0.490              &2.544                &2.957                  &1.866                      &10.433 \\
\texttt{DRR-5}                           &0.849                    &0.094                   &0.463               &0.625              &0.363              &0.773                &0.881                  &0.736                      &3.709 \\
\texttt{DRR-10}                          &0.623                    &0.073                   &0.298               &0.350              &0.214              &0.401                &0.503                  &0.435                      &2.645 \\
\texttt{MDD-LS-5}                        &0.875                    &0.115                   &0.587               &0.664              &0.427              &0.878                &1.167                  &0.876                      &4.774 \\
\texttt{MDD-LS-10}                       &0.656                    &0.084                   &0.315               &0.395              &0.269              &0.551                &0.628                  &0.452                      &3.156 \\
\hline \hline
\texttt{KRR}                             &3.450                    &1.508                   &9.801               &12.08              &76.99              &15.33                &16.103                 &137.6                      &/ \\
\texttt{KDRR-5}                          &1.487                    &0.295                   &3.374               &1.451              &5.524              &6.021                &5.913                  &40.22                      &86.754\\
\texttt{KDRR-10}                         &0.983                    &0.183                   &1.863               &0.689              &2.302              &3.670                &3.544                  &23.64                      &46.197\\
\texttt{MDD-RKHS-5}                      &1.692                    &0.331                   &5.637               &1.901              &7.854              &8.628                &7.454                   &53.09                      &103.20\\
\texttt{MDD-RKHS-10}                     &1.041                    &0.206                   &2.324               &0.884              &3.783             &4.125                &4.679                   &31.23                      &56.312\\
   \bottomrule
\end{tabular*}
\end{table*}
We experiment on 10 publicly available datasets from LIBSVM data
\footnote{Available at https://www.csie.ntu.edu.tw/~cjlin/libsvmtools/datasets/}.
We run all methods on a computer node with 32 cores (2.40GHz) and 64 GB memory.
While global methods only use a single CPU core, distributed methods use all cores to simulate parallel environment.
For RKHS methods, we use the popular Gaussian kernels $$K(\mathbf{x}, \mathbf{x}')=\exp\left(-\frac{\|\mathbf{x}-\mathbf{x}'\|_2^2}{2\sigma^2}\right)$$
as candidate kernels, and choose the best kernel from $\sigma \in \{2^i, i=-10, -9, \dots, 10\}$ by 5-folds cross-validation.
The regularized parameterized $\lambda \in \{10^i, i=-6,-5,\dots,3\}$ in all methods and $\gamma \in \{10^i, i=-6,-5,\dots,3\}$
in \texttt{MDD} methods are determined by 5-folds cross-validation on training data.
For each data set,  we run all methods 30 times with random partitions on all data
sets of non-overlapping 70\% training data and 30\% testing data.
All statements of statistical significance in the remainder refer to a 95\% level of significance under $t$-test.

The root mean square error of all methods is reported in Table \ref{tabel:mse}. Meanwhile,
we repeat distributed methods on different amount of worker nodes, 5 and 10 for simplification.
Table \ref{tabel:mse} can be summarized as follows:
\begin{itemize}
\item [1)] Our \texttt{MDD-LS} and \texttt{MDD-RKHS} exhibit better prediction accuracy than the $\texttt{DRR}$ and $\texttt{KDRR}$
over almost all data sets.
This demonstrates the advantage of \texttt{MDD} methods in generalization performance.
\item[2)] Our \texttt{MDD-LS} and \texttt{MDD-RKHS} give comparable result with global methods on most of data sets.
\item[3)] Kernel methods can usually get more optimal results than linear methods do;
\item[4)] Some data sets are sensitive to data partition, whose results existing huge gap between global methods and distributed methods,
such as space\_ga and phishing for RKHS, while others are not.
\item[5)] The increase of worker nodes causes higher root mean square error.
\end{itemize}

The running time is reported in Table \ref{tabel:time},
which can be summarized as follows:
\begin{itemize}
\item[1)] Global methods cost more time than distributed methods do on all data sets.
\item[2)] Kernel methods always spend more time than linear methods, because of higher computation complexity.
\item[3)] Distributed methods lead great speedup on some data sets.
\item[4)] The running time of distributed methods decays almost linearly associated with the increase of worker nodes.
\item[5)] Compared with global methods, our \texttt{MDD} methods own higher computational efficiency,
while existing small distance away from divide-and-conquer methods.
\end{itemize}

The above results show that \texttt{MDD} methods need a bit more training time
but make the performance gap between global methods and traditional distributed methods tighter,
which is consistent with our theoretical analysis.

\section{Conclusion}
In this paper, we studied the generalization performance of distributed learning,
and derived a sharper generalization error bound,
which is much sharper than existing  generalization bounds of divide-and-conquer based distributed learning.
Then, we designed two algorithms with statistical guarantees and fast convergence rates for linear space and RKHS:
\texttt{MDD-LS} and \texttt{MDD-RKHS}.
As we see from theoretical analysis and empirical results,
our \texttt{MDD} is highly competitive with the existing divide-and-conquer methods,
in terms of both practical performance and computational cost.
Based on max-diversity of each local estimate, our analysis can be used as a solid basis for the
design of new distributed learning algorithms.

\section{Proof}
\subsection{The Key Idea}
From the $\eta$-strongly convex of $R(f)$ of equation \eqref{assumption-strongly-second},
we can obtain that
\begin{align*}
  R(\bar{f})&=R\left(\frac{1}{m}\sum_{i=1}^m\hat{f}_i\right)\\&\leq
  \frac{1}{m}\sum_{i=1}^mR(\hat{f}_i)-\frac{\eta}{4m^2}\sum_{i,j=1, i\not=j}^m\|\hat{f}_i-\hat{f}_j\|_\mathcal{H}^2.
\end{align*}
Therefore, we have
\begin{align}
\begin{aligned}
\label{equaiton-strongly-ff}
  R(\bar{f})-R(f_\ast)\leq& \frac{1}{m}\sum_{i=1}^m\left[R(\hat{f}_i)-R(f_\ast)\right]
  \\&-\frac{\eta}{4m^2}\sum_{i,j=1,i\not=j}^m\|\hat{f}_i-\hat{f}_j\|_\mathcal{H}^2.
\end{aligned}
\end{align}

In the next, we will estimate $R(\hat{f}_i)-R(f_\ast)$,
which is built upon the following inequality from \eqref{assumption-strongly-equation}:
\begin{align}
\nonumber
  &R(\hat{f}_i)-R(f_\ast)+\frac{\eta}{2}\|\hat{f}_i-f_\ast\|_\mathcal{H}^2
  \\ &\leq\nonumber \langle \nabla R(\hat{f}_i),\hat{f}_i-f_\ast\rangle_\mathcal{H}
  \\ \nonumber &=\langle \nabla R(\hat{f}_i)-\nabla R(f_\ast)-[\nabla \hat{R}_i(\hat{f}_i)-\nabla \hat{R}_i(f_\ast)],
   \hat{f}_i-f_\ast\rangle_\mathcal{H}\\
   \label{equaiton-ddd}
   &~~~+\langle\nabla R(f_\ast)-\nabla \hat{R}_i(f_\ast), \hat{f}_i-f_\ast\rangle_\mathcal{H}+\langle \nabla \hat{R}_i(\hat{f}_i), \hat{f}_i-f_\ast\rangle_\mathcal{H}.
\end{align}
By the convexity of $\hat{R}_i(\cdot)$ and the optimality condition of $\hat{f}_i$ \cite{boyd2004convex},
we have
\begin{align}
  \label{optimal-empirical}
  \langle \nabla \hat{R}_i(\hat{f}_i),f-\hat{f}_i\rangle_\mathcal{H}\geq 0, \forall f\in\mathcal{H}.
\end{align}
Substituting \eqref{optimal-empirical} into \eqref{equaiton-ddd}, we have
\begin{align}
  \nonumber
  &~~~~~R(\hat{f}_i)-R(f_\ast)+\frac{\eta}{2}\|\hat{f}_i-f_\ast\|_\mathcal{H}^2\\
  &\nonumber\leq \langle \nabla R(\hat{f}_i)-\nabla R(f_\ast)-[\nabla \hat{R}_i(\hat{f}_i)-\nabla \hat{R}_i(f_\ast)],
  \hat{f}_i-f_\ast\rangle_\mathcal{H}\\&~~~~~+\langle\nabla R(f_\ast)-\nabla \hat{R}_i(f_\ast), \hat{f}_i-f_\ast\rangle_\mathcal{H}\\
   \label{equation-important-middle}
   &\begin{aligned}
   \leq \Big(\underbrace{\left\|\nabla R(\hat{f}_i)-\nabla R(f_\ast)-[\nabla \hat{R}_i(\hat{f}_i)-\nabla \hat{R}_i(f_\ast)]\right\|}_{:=A_1}\\+
    \underbrace{\left\|\nabla R(f_\ast)-\nabla \hat{R}_i(f_\ast)\right\|}_{=:A_2}\Big)\left\|\hat{f}_i-f_\ast \right\|.
    \end{aligned}
\end{align}
\subsection{Proof of Theorem \ref{theorem-main}}
To prove Theorem \ref{theorem-main}, we first give the following two lemmas (the proofs are given at the last part of this section).
\begin{lemma}
\label{lemma-nablaR-nablahat}
  Under \textbf{Assumptions} \ref{assumption-smooth-loss} and \ref{assumption-smooth-equaiton-loss}, with probability at least $1-\delta$,
  for any $f\in\mathcal{N}(\mathcal{H},\epsilon)$,
  we have
  \begin{align}
  \label{equation-nabalRR-empRR}
  \begin{aligned}
    &~~~~\left\|\nabla R(f)-\nabla R(f_\ast)-[\nabla \hat{R}_i(f)-\nabla \hat{R}_i(f_\ast)]\right\| \\
   &\leq \frac{(\tau+\tau')\|f-f_\ast\| \log C(\mathcal{H},\epsilon)}{n}\\&~~+\sqrt{\frac{(\tau+\tau')(R(f)-R(f_\ast)) \log C(\mathcal{H},\epsilon)}{n}}.
  \end{aligned}
  \end{align}
\end{lemma}
\begin{lemma}
\label{lemma-second-ff}
 Under \textbf{Assumptions} \ref{assumption-smooth-loss},
 with probability at least $1-\delta$, we have
  \begin{align}
  \label{equation-nablaR-nablaempR}
    \left\|\nabla R(f_\ast)-\nabla \hat{R}_i(f_\ast)\right\|\leq \frac{2M\log(2/\delta)}{n}+\sqrt{\frac{8\tau H_\ast\log(2/\delta)}{n}},
  \end{align}
  where $H_\ast=\mathbb{E}_{z}\left[\ell(f_\ast,z)\right]$.
\end{lemma}


\begin{proof}[Proof of Theorem \ref{theorem-main}]
From the property of $\epsilon$-net, we know that there exists a point
$\tilde{f}\in\mathcal{N}(\mathcal{H},\epsilon)$ such that $$\|\hat{f}_i-\tilde{f}\|\leq \epsilon.$$
According to \textbf{Assumptions} \ref{assumption-smooth-loss} and \ref{assumption-smooth-r},
we know that $R(f)$ and $\hat{R}(f)$ are both $(\tau+\tau')$-smooth.
Thus, we have
\begin{align}
  \nonumber
    &~~~\left\|\nabla R(\hat{f}_i)-\nabla R(f_\ast)-[\nabla \hat{R}_i(\hat{f}_i)-\nabla \hat{R}_i(f_\ast)]\right\| \\
   \nonumber &\leq \left\|\nabla R(\tilde{f})-\nabla R(f_\ast)-[\nabla \hat{R}_i(\tilde{f})-\nabla \hat{R}_i(f_\ast)]\right\|+2(\tau+\tau') \epsilon\\
   \nonumber &\overset{\eqref{equation-nabalRR-empRR}}{\leq}
   \frac{(\tau+\tau') \log C(\mathcal{H},\epsilon)\|\tilde{f}-f_\ast\|}{n}\\ \nonumber &~~~+
   \sqrt{\frac{(\tau+\tau') \log C(\mathcal{H},\epsilon)(R(\tilde{f})-R(f_\ast))}{n}}+2(\tau+\tau')\epsilon\\
   \nonumber&\leq \frac{(\tau+\tau') \log C(\mathcal{H},\epsilon)\|\hat{f}_i-f_\ast\|_\mathcal{H}}{n}
   \\&~~~\nonumber+\frac{(\tau+\tau') \log C(\mathcal{H},\epsilon)\epsilon}{n}+2(\tau+\tau')\epsilon\\
   \nonumber&~~~+\sqrt{\frac{\beta \log C(\mathcal{H},\epsilon)(R(\hat{f}_i)-R(f_\ast))}{n}}\\ \nonumber&~~~+
   \sqrt{\frac{(\tau+\tau') \log C(\mathcal{H},\epsilon)\left(\left|R(\hat{f}_i)-R(\tilde{f})\right|\right)}{n}}\\
   \nonumber &\overset{\eqref{assumption-libs-equation}}{\leq} \frac{(\tau+\tau') \log C(\mathcal{H},\epsilon)\|\hat{f}_i-f_\ast\|_\mathcal{H}}{n}\\\nonumber&~~~+
   \frac{(\tau+\tau') \log C(\mathcal{H},\epsilon)\epsilon}{n}+2(\tau+\tau')\epsilon\\
   \label{equation-31}&~~~+\sqrt{\frac{(\tau+\tau') \log C(\mathcal{H},\epsilon)(R(\hat{f}_i)-R(f_\ast))}{n}}\\\nonumber&~~~+
   \sqrt{\frac{(\tau+\tau') L \log C(\mathcal{H},\epsilon)\epsilon}{n}}
\end{align}

  Substituting \eqref{equation-31} and \eqref{equation-nablaR-nablaempR} into \eqref{equation-important-middle},
  with probability at least $1-2\delta$, we have
  \begin{align}
    \label{equation-first-equation}
    \begin{aligned}
      &R(\hat{f}_i)-R(f_\ast)+\frac{\eta}{2}\|\hat{f}_i-f_\ast\|_\mathcal{H}^2\\
      &\leq \frac{(\tau+\tau') \log C(\mathcal{H},\epsilon)\|\hat{f}_i-f_\ast\|_\mathcal{H}^2}{n}\\&~~~+
   \frac{(\tau+\tau') \log C(\mathcal{H},\epsilon)\epsilon \|\hat{f}_i-f_\ast\|_\mathcal{H}}{n}\\&~~~+2(\tau+\tau')\epsilon \|\hat{f}_i-f_\ast\|_\mathcal{H}\\
   &~~~+\|\hat{f}_i-f_\ast\|_\mathcal{H}\sqrt{\frac{(\tau+\tau') \log C(\mathcal{H},\epsilon)(R(\hat{f}_i)-R(f_\ast))}{n}}\\&~~~+
   \|\hat{f}_i-f_\ast\|_\mathcal{H}\sqrt{\frac{(\tau+\tau') L \log C(\mathcal{H},\epsilon)\epsilon}{n}}\\
   &~~~~+\frac{2M \log(2/\delta)\|\hat{f}_i-f_\ast\|_\mathcal{H}}{n}\\
   &~~~+\|\hat{f}_i-f_\ast\|_\mathcal{H}\sqrt{\frac{8\tau H_\ast \log(2/\delta)}{n}}.
    \end{aligned}
  \end{align}
  Note that
  \begin{align*}
    \sqrt{ab}\leq \frac{a}{2c}+\frac{bc}{2}, \forall a,b,c\geq 0.
  \end{align*}
  Therefore, we can obtain that
  \begin{align*}
    &~~~\|\hat{f}_i-f_\ast\|_\mathcal{H}\sqrt{\frac{(\tau+\tau') \log C(\mathcal{H},\epsilon)(R(\hat{f}_i)-R(f_\ast))}{n}}
    \\&\leq \frac{2(\tau+\tau') \log C(\mathcal{H},\epsilon)(R(\hat{f}_i)-R(f_\ast))}{n\eta}+\frac{\eta }{8}\|\hat{f}_i-f_\ast\|_\mathcal{H}^2;\\
    &~~~\frac{2M \log(2/\delta)\|\hat{f}_i-f_\ast\|_\mathcal{H}}{n}\\&\leq \frac{8M \log(2/\delta)}{n^2\eta}+\frac{\eta }{16}\|\hat{f}_i-f_\ast\|_\mathcal{H}^2;\\&
    ~~~\|\hat{f}_i-f_\ast\|_\mathcal{H}\sqrt{\frac{8\eta H_\ast \log(2/\delta)}{n}}\\&\leq \frac{64\eta H_\ast\log(2/\delta)}{n\eta}+\frac{\eta }{32}\|\hat{f}_i-f_\ast\|_\mathcal{H}^2;\\
    &~~~2(\tau+\tau')\epsilon \|\hat{f}_i-f_\ast\|_\mathcal{H}\\&\leq \frac{32(\tau+\tau')^2\epsilon^2}{\eta}+\frac{\eta }{64}\|\hat{f}_i-f_\ast\|_\mathcal{H}^2;\\&
    ~~~\|\hat{f}_i-f_\ast\|_\mathcal{H}\sqrt{\frac{(\tau+\tau') L \log C(\mathcal{H},\epsilon)\epsilon}{n}}\\&
    \leq \frac{32(\tau+\tau') L \log C(\mathcal{H},\epsilon)\epsilon}{n\eta}+\frac{\eta }{128}\|\hat{f}_i-f_\ast\|_\mathcal{H}^2;\\&
    ~~~\frac{(\tau+\tau') \log C(\mathcal{H},\epsilon)\epsilon \|\hat{f}_i-f_\ast\|_\mathcal{H}}{n}
    \\&\leq \frac{32(\tau+\tau') \log^2 C(\mathcal{H},\epsilon)\epsilon^2}{n^2\eta}+\frac{\eta }{128}\|\hat{f}_i-f_\ast\|_\mathcal{H}^2.
  \end{align*}
 Substituting the above  inequation into \eqref{equation-first-equation}, we can obtain that
 \begin{align*}
   &~~~~R(\hat{f}_i)-R(f_\ast)+\frac{\eta}{4}\|\hat{f}_i-f_\ast\|_\mathcal{H}^2\\
   &\leq \frac{(\tau+\tau') \log C(\mathcal{H},\epsilon)\|\hat{f}_i-f_\ast\|_\mathcal{H}^2}{n}\\&~~~+
   \frac{2(\tau+\tau') \log C(\mathcal{H},\epsilon)(R(\hat{f}_i)-R(f_\ast))}{n\eta}+\frac{8M \log(2/\delta)}{n^2\eta}\\
   &~~~+\frac{64\tau H_\ast\log(2/\delta)}{n\eta}+\frac{32(\tau+\tau')^2\epsilon^2}{\eta}\\&~~~+
   \frac{32(\tau+\tau') L \log C(\mathcal{H},\epsilon)\epsilon}{n\eta}\\&~~~+
   \frac{32(\tau+\tau') \log^2 C(\mathcal{H},\epsilon)\epsilon^2}{n^2\eta}\\
   &\overset{\eqref{equation-12}}{\leq}
   \frac{\eta}{4}\|\hat{f}_i-f_\ast\|_\mathcal{H}^2+\frac{1}{2}(R(\hat{f}_i)-R(f_\ast))+\frac{8(\tau+\tau') \log(2/\delta)}{n^2\eta}\\
   &~~~+\frac{64\tau H_\ast\log(2/\delta)}{n\eta}+\frac{32(\tau+\tau')^2\epsilon^2}{\eta}\\&~~~+
   \frac{32(\tau+\tau') L \log C(\mathcal{H},\epsilon)\epsilon}{n\eta}\\&~~~+
   \frac{32(\tau+\tau') \log^2 C(\mathcal{H},\epsilon)\epsilon^2}{n^2\eta}.
 \end{align*}
 Thus, with $1-2\delta$, we have
 \begin{align}
    \label{equation-14}
    \begin{aligned}
    &~~~R(\hat{f}_i)-R(f_\ast)\\&\leq
    \frac{16M \log(2/\delta)}{n^2\eta}+\frac{128\tau H_\ast\log(2/\delta)}{n\eta}\\&~~~+\frac{32(\tau+\tau')^2\epsilon^2}{\eta}+
   \frac{64(\tau+\tau') L \log C(\mathcal{H},\epsilon)\epsilon}{n\eta}\\&~~~+
   \frac{64(\tau+\tau') \log^2 C(\mathcal{H},\epsilon)\epsilon^2}{n^2\eta}.
   \end{aligned}
  \end{align}
  Combining \eqref{equaiton-strongly-ff} and \eqref{equation-14},
  with $1-\delta$,
  we have
  \begin{align*}
    &~~~~~R(\bar{f})-R(f_\ast)\\&~~~\leq
    \frac{16M \log(4m/\delta)}{n^2\eta}+\frac{128\tau H_\ast\log(4m/\delta)}{n\eta}\\&~~~+\frac{32(\tau+\tau')^2\epsilon^2}{\eta}+
   \frac{64(\tau+\tau') L \log C(\mathcal{H},\epsilon)\epsilon}{n\eta}\\&~~~+
   \frac{64(\tau+\tau') \log^2 C(\mathcal{H},\epsilon)\epsilon^2}{n^2\eta}\\&~~~
  -\frac{\eta}{4m^2}\sum_{i,j=1,i\not=j}^m\|\hat{f}_i-\hat{f}_j\|_\mathcal{H}^2.
  \end{align*}
\end{proof}

\subsection{Proof of Lemma \ref{lemma-nablaR-nablahat}}
\begin{lemma}[\cite{smale2007learning}]
  \label{lem-nonequation}
    Let $\mathcal{H}$ be a Hilbert space and let $\xi$ be a random variable with values in $\mathcal{H}$.
    Assume $\|\xi\|\leq M\leq \infty$ almost surely.
    Denote $\sigma^2(\xi)=\mathbb{E}[\|\xi\|^2]$.
    Let $\{\xi_i\}_{i=1}^n$ be $m$ independent drawers of $\xi$.
    For any $0\leq \delta\leq 1$, with confidence $1-\delta$,
    \begin{align*}
      \left\|\frac{1}{n}\sum_{j=1}^n[\xi_j-\mathbb{E}[\xi_j]]\right\|\leq \frac{2M\log(2/\delta)}{n}+\sqrt{\frac{2\sigma^2(\xi)\log(2/\delta)}{n}}.
    \end{align*}
  \end{lemma}
\begin{proof}
  According to \textbf{Assumption} \ref{assumption-smooth-loss} and \ref{assumption-smooth-equaiton-loss},
  we know that $\nu(f,\cdot)=\nu(f,z)=\ell(f,z)+r(f)$ is $(\tau+\tau')$-smooth,
  so we have
  \begin{align}
    \|\nabla\nu(f,\cdot)-\nabla\nu(f_\ast,\cdot)\|_\mathcal{H} \leq (\tau+\tau')\|f-f_\ast\|_\mathcal{H}
  \end{align}
  Because $\nu(f,\cdot)$ is $(\tau+\tau')$-smooth and convex, by (2.1.7) of \cite{Nesterov2004},
  $\forall z\in\mathcal{Z}$,
  we have
  \begin{align*}
   &~~~\left\|\nabla\nu (f,z)-\nabla \nu (f_\ast,z)\right\|^2
    \\&\leq
    (\tau+\tau')\left(\nu (f,z)-\nu (f_\ast,z)
    -\langle \nabla\nu (f_\ast,z), f-f_\ast\rangle_\mathcal{H} \right).
  \end{align*}
  Taking expectation over both sides, we have
  \begin{align*}
    &~~~~\mathbb{E}_{z\sim\mathbb{P}}[\left\|\nabla\nu (f,\cdot)-\nabla\nu (f_\ast,\cdot)\right\|^2]\\
    &\leq (\tau+\tau') \left(R(\hat{f}_i)-R(f_\ast)-\langle \nabla R(f_\ast), f-f_\ast\rangle_\mathcal{H}\right)\\
    &\leq (\tau+\tau') \left(R(\hat{f}_i)-R(f_\ast)\right)
  \end{align*}
  where the last inequality follows from the optimality condition of $f_\ast$, i.e.,
  \begin{align*}
    \langle \nabla R(f_\ast),f-f_\ast\rangle_\mathcal{H} \geq 0,\forall f\in\mathcal{H}.
  \end{align*}

Following Lemma \ref{lem-nonequation}, with probability at least $1-\delta$, we have
\begin{align*}
  &~~~~\left\|
  \nabla R(f)-\nabla R(f_\ast)-[\nabla \hat{R}_i(f)-\nabla \hat{R}_i(f_\ast)]
  \right\|_\mathcal{H}\\
  &=\left\|
    \nabla R(f)-\nabla R(f_\ast)
    -\frac{1}{n}\sum_{z_i\in \mathcal{S}_i}
    \left[\nabla \nu(f,z_i)-\nabla \nu(f_\ast,z_i)\right]
  \right\|_\mathcal{H}\\
  &\leq \frac{2(\tau+\tau')\|f-f_\ast\|_\mathcal{H}\log(2/\delta)}{n}\\&
  ~~~+\sqrt{\frac{2(\tau+\tau')(R(f)-R(f_\ast))\log (2/\delta)}{n}}.
\end{align*}

We obtain Lemma \ref{lemma-nablaR-nablahat} by taking the union bound over all $f\in\mathcal{N}(\mathcal{H},\epsilon)$.
\end{proof}
\subsection{Appendix: Proof of Lemma \ref{lemma-second-ff}}
\begin{proof}
  Since $\ell(f,\cdot)$ is $\eta$-smooth and nonegative,
  from Lemma 4 of \cite{srebro2010optimistic}, we have
  \begin{align*}
    \left\|\nabla \ell(f_\ast,z_i)\right\|^2\leq 4(\tau+\tau') \ell(f_\ast,z_i)
  \end{align*}
  and thus
    \begin{align*}
      \mathbb{E}_{z\sim\mathbb{P}}\left[\left\|\nabla \ell(f_\ast,z)\right\|^2\right]&\leq 4(\tau+\tau')\mathbb{E}_{z\sim\mathbb{P}}[\ell(f_\ast,z)]\\&=
      4(\tau+\tau') R(f_\ast).
    \end{align*}
    From the \textbf{Assumption}, we have $\nabla \|\ell(f_\ast, z)\|\leq M$, $\forall z\in\mathcal{Z}$.
    Let $H(f)=R(f)-r(f)$ and $\hat{H}(f)=\hat{R}(f)-r(f)$.
    Then, according to Lemma \ref{lem-nonequation}, with probability at least $1-\delta$, we have
    \begin{align*}
      &~~~\left\|\nabla R(f_\ast)-\nabla \hat{R}_i(f_\ast)\right\|=\left\|\nabla H(f_\ast)-\nabla \hat{H}_i(f_\ast)\right\|
      \\&=
      \left\|\nabla H(f_\ast)-\frac{1}{n}\sum_{z_j\in\mathcal{S}_i}\nabla \ell(f_\ast,z_j)\right\|\\
      &\leq \frac{2(\tau+\tau') \log(2/\delta)}{n}+\sqrt{\frac{8(\tau+\tau') H_\ast \log(2/\delta)}{n}}.
    \end{align*}
\end{proof}
\subsection{Proof of Lemma \ref{lemma-fast-linear-space}}
\begin{lemma}
  \label{lemma-fast-linear-space}
  For all $\ell\geq 1$, If $\mathbf A\in\mathbb{R}^{l\times l}$ is a symmetric matrix and $\mathbf b, \mathbf d\in\mathbb{R}^l$,
  $\mathbf c=\mathbf A^{-1}\mathbf b\in\mathbb{R}^l$,
  then we have
  \begin{align*}
  \mathbf A^{-1}\mathbf d=(\mathbf d^\mathrm{T}\mathbf c)./\mathbf b,
  \end{align*}
  where $a./\mathbf c=(a/c_1,\ldots a/c_l)^\mathrm{T}$.
\end{lemma}
\begin{proof}
Since $\mathbf A$ a symmetric matrix,
we have
  \begin{align*}
    \left(\mathbf A^{-1}\mathbf d\right)^\mathrm{T}\mathbf b=\mathbf d^\mathrm{T}\mathbf A^{-1}\mathbf b=\mathbf d^\mathrm{T}\mathbf c.
  \end{align*}
Therefore, we can obtain that $\mathbf A^{-1}\mathbf d=(\mathbf d^\mathrm{T}\mathbf c)./\mathbf b$.
\end{proof}
%
%
%
%
%
\bibliographystyle{IEEEtran}
\bibliography{NIPS2018}
\end{document}